\documentclass[letterpaper]{article} 
\usepackage{aaai2026}  
\usepackage{times}  
\usepackage{helvet}  
\usepackage{courier}  
\usepackage[hyphens]{url}  
\usepackage{graphicx} 
\usepackage{natbib}  
\usepackage{caption} 
\usepackage{algorithm}
\usepackage{algorithmic}
\usepackage{tabularx}
\usepackage{booktabs}
\usepackage{subcaption}

\usepackage{newfloat}
\usepackage{listings}
\usepackage{amsthm}
\usepackage{amsmath}
\usepackage{amsfonts}
\usepackage{amssymb}
\usepackage{multirow}
\usepackage{colortbl}
\newtheorem{assumption}{Assumption}
\newtheorem{lemma}{Lemma}
\newtheorem{definition}{Definition}

\newtheorem{theorem}{Theorem}

\DeclareCaptionStyle{ruled}{labelfont=normalfont,labelsep=colon,strut=off} 
\floatstyle{ruled}
\newfloat{listing}{tb}{lst}{}
\floatname{listing}{Listing}
\title{Enhancing Control Policy Smoothness by\\ Aligning Actions with Predictions from Preceding States}

\author{
    Kyoleen Kwak and Hyoseok Hwang\thanks{Corresponding author (hyoseok@khu.ac.kr)}
}
\affiliations{
    Department of Artificial intelligence, Kyung Hee University, Republic of Korea\\

    \{2007kkl,hyoseok\}@khu.ac.kr
}

\usepackage{bibentry}

\begin{document}

\maketitle

\begin{abstract}
Deep reinforcement learning has proven to be a powerful approach to solving control tasks, but its characteristic high‑frequency oscillations make it difficult to apply in real‑world environments.
While prior methods have addressed action oscillations via architectural or loss-based methods, the latter typically depend on heuristic or synthetic definitions of state similarity to promote action consistency, which often fail to accurately reflect the underlying system dynamics.
In this paper, we propose a novel loss-based method by introducing a transition-induced similar state.
The transition-induced similar state is defined as the distribution of next states transitioned from the previous state.
Since it utilizes only environmental feedback and actually collected data, it better captures system dynamics.
Building upon this foundation, we introduce Action Smoothing by Aligning Actions with Predictions from Preceding States (ASAP), an action smoothing method that effectively mitigates action oscillations. 
ASAP enforces action smoothness by aligning the actions with those taken in transition-induced similar states and by penalizing second-order differences to suppress high-frequency oscillations.
Experiments in Gymnasium and Isaac-Lab environments demonstrate that ASAP yields smoother control and improved policy performance over existing methods.
\end{abstract}

\begin{links}
    \link{Code}{https://github.com/AIRLABkhu/ASAP}
\end{links}

\section{Introduction}
Reinforcement learning (RL) has seen tremendous advances in recent years \cite{9904958}. 
Beginning with discrete methods, RL has evolved into Deep Reinforcement Learning (DRL) through the integration of deep neural networks and has been extended to continuous action spaces, becoming a leading solution for control problems \cite{williams1992simple, lillicrap2015continuous}. 
Building on these advances, RL is widely applied to robot and machine control with impressive results \cite{levine2016end}, and these efforts aim to deploy learned policies in real-world settings \cite{8202133}.

Despite these successes, deploying RL in real-world settings remains challenging, with high-frequency action oscillations posing a significant obstacle.
When combined with hardware constraints in the real-world, such action oscillations can dramatically reduce component lifespan and cause excessive wear \cite{DELAPRESILLA2023108805}. 
In addition, action oscillations produce non-smooth actions, and these non-smooth actions can directly cause poor user experience, or in severe cases, safety problems.
This issue arises because the actor is overly sensitive to minor state perturbations, which yields large action deviations \cite{chen2021addressing}.

Research on mitigating action oscillations in RL has primarily focused on achieving smoother policy outputs by regulating output variation with respect to input fluctuations. Existing approaches fall into two main categories: architectural methods and loss penalty methods.
Architectural methods aim to prevent abrupt output changes by modifying the network structure, typically applying Spectral Normalization~\cite{takase2022stability} or Multi-dimensional Gradient Normalization~\cite{song2023lipsnet} to the actor network. 
However, these methods introduce additional computational cost during inference and often suffer from significant performance variability across environments~\cite{christmann2024benchmarking}.

On the other hand, loss penalty methods incorporate oscillation regularization terms directly into the policy loss. 
This approach does not alter the actor structure and thus avoids extra inference-time computation. A key idea behind this method is to enforce action consistency across similar states, making the definition of similar states crucial. 
Initial approaches employed fixed distributions~\cite{mysore2021regularizing}, while more recent methods introduced adaptive boundaries~\cite{kobayashi2022l2c2}. 
However, these approaches rely on heuristic neighborhood definitions and fill the region with synthetic rather than observed states.
As a result, the similar states do not accurately reflect the actual state distribution.
This mismatch undermines the theoretical guarantees and leads to significant declines in performance.

To overcome these limitations, in this paper, we propose a novel loss penalty based RL action smoothing method called regulate action oscillation via \textbf{A}ction \textbf{S}moothing by \textbf{A}ligning Actions with \textbf{P}redictions from Preceding States (ASAP).
ASAP introduces a transition-induced definition of similar states, under the assumption that states transitioning from the same previous state are similar.
We theoretically prove that this definition induces a bounded neighborhood, thereby providing a principled foundation for action smoothing based on similar states. 
Building on this foundation, ASAP introduces a spatial term that constrains action changes.
This spatial term constrains action changes across similar states by reducing the discrepancy between a predicted next action from the preceding state and the actual action taken.
In addition, ASAP adopts a temporal term that penalizes second-order differences to suppress high-frequency oscillations \cite{lee2024gradient}. 
Unlike prior methods that rely on synthetic or heuristically defined neighborhoods, our approach is trained exclusively on empirically collected transitions, ensuring fidelity to the real state distribution and alignment with system dynamics. 

We evaluated ASAP's policy performance and oscillation in Gymnasium \cite{towers2024gymnasium}, and Isaac-Lab environments \cite{mittal2023orbit}. 
Through experiments with various RL algorithms, we confirmed that ASAP effectively reduces action oscillations across multiple tasks while preserving the performance compared to other methods. 

This work's main contributions are as follows:
\begin{itemize}
\item We introduce \textbf{ASAP}, a new action smoothing method that combines the action constraint using transition-induced similar state with a penalty on second-order action variations.
\item We define similar states derived from the maximum-change bound of the transition dynamics, and theoretically prove that they form a bounded neighborhood for action smoothing.
\item Extensive experiments demonstrate that ASAP effectively suppresses action oscillations while preserving policy performance over existing methods.
\end{itemize}

\section{Related Work}
Many studies have focused on reducing the Lipschitz constant of the actor network to constrain its output changes to mitigate the action oscillations. 
Related studies are broadly categorized into two approaches:
\subsubsection{Architectural Methods.}
Architectural methods alter the network itself to satisfy Lipschitz constraints. 
Gogianu et al. showed that Spectral Normalization can stabilize reinforcement learning by controlling Lipschitz smoothness \cite{gogianu2021spectral}.
LipsNet \cite{song2023lipsnet} enforces global and local Lipschitz continuity via Multi-dimensional Gradient Normalization by regularizing the Jacobian's operator norm.
LipsNet++ \cite{song2025lipsnet} adds a Fourier filter and a Lipschitz controller with Jacobian regularization to decouple noise and smoothness.

\subsubsection{Loss Penalty Methods.}
Loss penalty methods add Lipschitz regularizers directly to the RL loss. 
Conditioning for Action Policy Smoothness (CAPS) \cite{mysore2021regularizing} adds temporal smoothing between adjacent states and spatial smoothing over a fixed range neighborhood of similar states to limit actor output variation.
Local Lipschitz Continuous Constraint (L2C2) \cite{kobayashi2022l2c2} dynamically adjusts the range of similar states based on the current state, thereby implementing the concept of local Lipschitz continuity in reinforcement learning.
Gradient-based CAPS (Grad-CAPS) \cite{lee2024gradient} proposed a new approach by applying Lipschitz constraints on derivatives instead of conventional Lipschitz constraints from a temporal perspective.
All of these methods share the network architecture of the base model and only increase its loss, so they incur additional computation during training while performing operations identical to the base model in inference.

\begin{figure*}[t]
  \centering
  \begin{subfigure}[t]{0.75\textwidth}
    \includegraphics[width=\columnwidth]{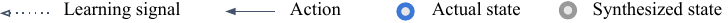}
    \label{fig:sim_desc}
  \end{subfigure}\\[-2.0ex]
  \begin{subfigure}[t]{0.33\textwidth}
    \includegraphics[width=\columnwidth]{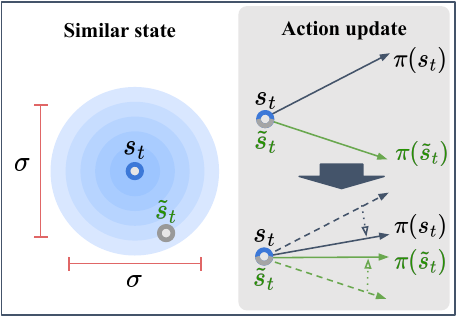}
    \caption{CAPS}
    \label{fig:sim_caps}
  \end{subfigure}
  \begin{subfigure}[t]{0.33\textwidth}
    \includegraphics[width=\columnwidth]{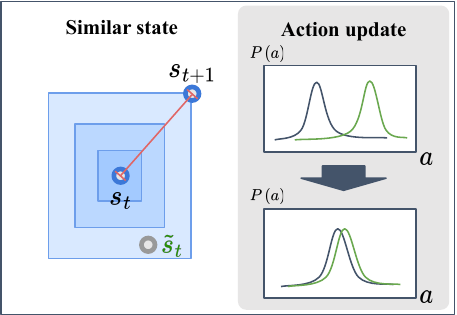}
    \caption{L2C2}
    \label{fig:sim_l2c2}
  \end{subfigure}
  \begin{subfigure}[t]{0.33\textwidth}
    \includegraphics[width=\columnwidth]{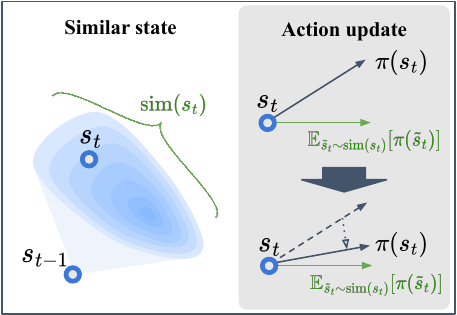}
    \caption{ASAP}
    \label{fig:sim_asap}
  \end{subfigure}
\caption{
Definition of similar states and action update mechanisms for (a) CAPS, (b) L2C2, and (c) ASAP.
(a) CAPS samples states from a Gaussian around the current state and enforces action consistency between real and sampled states.
(b) L2C2 constructs similar states by projecting the difference between the next state and the current state in multiple directions and minimizes divergence between the similar and actual policy distributions.
(c) ASAP uses the environment's transition distribution to define similar states and aligns the action with the expected policy output under this distribution.
}
  \label{fig:simstate}
\end{figure*}

\section{Background}
\subsection{Lipschitz Continuity}
\begin{definition}[Global Lipschitz Continuity]
A function \(f: \mathcal{X} \to \mathcal{Y}\) is called Lipschitz continuous with global Lipschitz constant \(K_{\mathrm{glob}}\in\mathbb{R}_+\) if, for all \(x^{(1)}, x^{(2)} \in \mathcal{X}\),
\begin{equation}\label{eq:global_lips}
d_Y(f(x^{(1)}),\; f(x^{(2)})) \;\le\; K_{\mathrm{glob}}\,d_X(x^{(1)},\; x^{(2)}),
\end{equation}
\end{definition}

\begin{definition}[Local Lipschitz Continuity]
A function \(f: \mathcal{X} \to \mathcal{Y}\) is locally Lipschitz continuous if, for every point \(x_i\in\mathcal{X}\), there exist a neighborhood \(U_i\subseteq\mathcal{X}\) and a local Lipschitz constant \(K_{\mathrm{loc}}(U_i)\in\mathbb{R}_+\) such that
\begin{equation}\label{eq:local_lips}
\begin{aligned}
d_Y(f(x_i^{(1)}),\; f(x_i^{(2)}))\;\le\;K_{\mathrm{loc}}(U_i)\,d_X(x_i^{(1)},\; x_i^{(2)}),
\\
\forall\,x_i^{(1)},\; x_i^{(2)}\in U_i,
\end{aligned}
\end{equation}
\end{definition}

\noindent here, \(d_X\) and \(d_Y\) are the distance metrics on the spaces \(\mathcal{X}\) and \(\mathcal{Y}\), respectively.

Lipschitz constraint means constraining the Eq.~\ref{eq:global_lips} or Eq.~\ref{eq:local_lips} to be satisfied. 
A lower Lipschitz constant \(K\) (whether \(K_{\mathrm{glob}}\) or \(K_{\mathrm{loc}}\)) implies smaller output variations for given input changes, enhancing the robustness to noise and resulting in smoother outputs. 
Since a single global constant cannot optimally fit all regions of the input space, a global Lipschitz constraint can prevent reaching the optimal solution. 

\subsection{Lipschitz Constraints in Reinforcement Learning}
This section describes how Lipschitz constraints have been applied in reinforcement learning to achieve action smoothness. 
After CAPS decomposed Lipschitz constraints for policy into temporal and spatial constraints, most subsequent works have extended their methods based on these two axes.
\subsubsection{Temporal Smoothness.}
For every timestep \(t\) in an episode, for each consecutive input sequence, the network's outputs satisfy the Lipschitz constraint such as
\begin{equation}
    d_A(f(s_{t}) ,\; f(s_{t+1}))\;\le\;K\,d_T(s_{t} ,\; s_{t+1}) \;=\;K,
\end{equation}
\noindent where \(s_t\in\mathcal{S}\) is the state, \(a_t=f(s_t)\in\mathcal{A}\) is the action, and \(f\colon\mathcal{S}\to\mathcal{A}\) is the actor. 
Here, \(d_A\) measures distance in the action space \(\mathcal{A}\), and \(d_T\) is the temporal distance metric.
In reinforcement learning, since the intervals between timesteps are typically assumed to be uniform, the temporal distance metric \(d_T\) is fixed to 1.

\subsubsection{Spatial Smoothness.}
For each state \(s_t\), the network satisfies a Lipschitz constraint:
\begin{equation}
    d_A(f(s_t),\; f(\tilde{s}_t)) \leq K\, d_S(s_t,\; \tilde{s}_t),
\end{equation}
where \(d_A\) and \(d_S\) denote distances in the action and state spaces, and \(\tilde{s}_t\) is a similar state. 
The definition of \(\tilde{s}_t\) varies by method (Figure~\ref{fig:simstate}): CAPS samples from \(\mathcal{N}(s_t, \sigma)\), enforcing global Lipschitz constraint, while L2C2 sets \(\tilde{s}_t = s_t + (s_{t+1} - s_t)\cdot u\) to reflect local Lipschitz constraint. 
However, because neither definition matches the true state distribution, additional synthetic states must be generated to populate the similar state region.

\section{Method}
\subsection{Overview}
We define transition-induced similar states via the transition function and demonstrate that they form a bounded neighborhood, providing a basis for spatial regularization. 
This yields a Lipschitz constraint on the composite function \(f \circ T\), from which a spatial loss for the actor is derived. 
Based on this foundation, we propose ASAP, which (i) enforces action consistency through a spatial loss and (ii) penalizes second-order action differences via a temporal loss. 
We also describe the training procedure of our method.

\begin{figure}[t]
  \centering
  \begin{subfigure}[t]{0.45\textwidth}
    \includegraphics[width=\columnwidth]{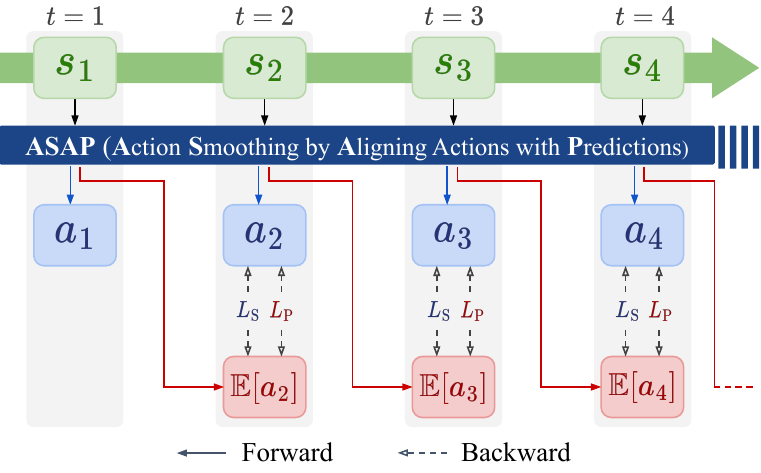}
    \caption{}
    \label{fig:update_procedure}
  \end{subfigure}
  \hspace{0.05\textwidth} 
  \begin{subfigure}[t]{0.45\textwidth}
    \includegraphics[width=\columnwidth]{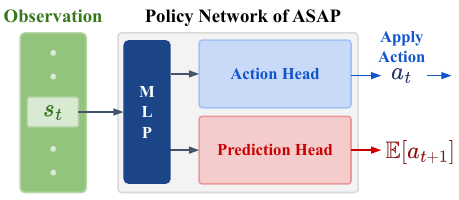}
    \caption{}
    \label{fig:arch}
  \end{subfigure}
  \caption{
  (a) The update procedure of ASAP and (b) The implementation architecture of ASAP.
}
  \label{fig:arch_all}
\end{figure}

\subsection{Similar State Based on Transition Distribution}
We define the distribution of the next states transitioning from the same previous state \(s_{t}\) as the distribution of similar states, as shown in Figure~\ref{fig:sim_asap}.

\begin{definition}[Similar State Distribution]\label{df:sim_state}
Given a state \(s_{t}\), the distribution of similar states is defined as
\begin{equation}\label{eq:sim_state}
\mathrm{sim}(s_{t}) \;=\; P\bigl(\,\cdot\mid s_{t-1}\bigr),
\end{equation}
where \(P(\,\cdot\mid s_{t-1})\) denotes the transition distribution from \(s_{t-1}\).
\end{definition}

To show that the above definition yields a locally bounded neighborhood suitable for imposing a local Lipschitz constraint, we introduce the following assumptions.

\begin{assumption}[Local Lipschitz Continuity of the Transition Function with respect to Noise]\label{ass:transition_lips}
For the state transition function
\begin{equation}\label{eq:noisy_transition}
s_{t} = T(s_{t-1}, a_{t-1}, \xi_{t-1}),
\end{equation}
we assume that, for each state-action pair \((s_{t-1},a_{t-1})\), the transition function is Lipschitz continuous with respect to the noise \(\xi_t\in\Xi\) acting on that pair, i.e.,
\begin{equation}\label{eq:transition_lips}
\begin{aligned}
d_S(T(s_{t-1},\;a_{t-1},\;\xi_{t-1}^{(1)}) ,\; T(s_{t-1},\;a_{t-1},\;\xi_{t-1}^{(2)})) \\ 
\le K_\xi(s_{t-1},\;a_{t-1})\,d_{\Xi}(\xi_{t-1}^{(1)} ,\; \xi_{t-1}^{(2)}),
\end{aligned}
\end{equation}
where \(K_\xi:\mathcal S\times\mathcal A\to\mathbb R_+\) is the local noise‑sensitivity constant at \((s_{t-1},\;a_{t-1})\), and \(d_{\Xi}\) measures distance in the noise space \(\Xi\).
\end{assumption}
The real-world dynamics corresponding to the transition function \(T(\cdot)\) in Assumption~\ref{ass:transition_lips} are typically \(C^1\) under most conditions, and this \(C^1\) property is commonly assumed in many system and controller designs~\cite{spong2006robot, featherstone2008rigid, camacho2013mpc}.
It is also well known that any \(C^1\) function is locally Lipschitz continuous~\cite{sontag1998mct}.

\begin{assumption}[Bounded Noise]\label{ass:bounded_noise}
The noise \(\xi_t\) satisfies a hard bound \(d_{\Xi}(\xi_t,\;0)\le \sigma_\xi\).
\end{assumption}
Although real-world noise is not strictly bounded, sensor filtering, digital clamping, and physical constraints keep it effectively within limits, which underpins many robust/optimal control theories~\cite{572756, camacho2013mpc}.

\begin{lemma}[Spatially Bounded Neighborhood]\label{lem:radius}
Under Assumptions~\ref{ass:transition_lips}–\ref{ass:bounded_noise},
any two states \(s_{t}^{(1)},\; s_{t}^{(2)}\sim\mathrm{sim}(s_{t})\) satisfy
\begin{equation}\label{eq:bounded_neighbor}
  d_S(s_{t}^{(1)} ,\; s_{t}^{(2)})\le 2K_\xi(s_{t-1},\;a_{t-1})\sigma_\xi.
\end{equation}
\end{lemma}

\begin{proof}
By the Lipschitz continuity assumption on the transition function,
\begin{equation}\label{eq:transition_lips_proof}
\begin{split}
  &d_S(s_{t}^{(1)} ,\; s_{t}^{(2)}) \\
  &= d_S(T(s_{t-1},\;a_{t-1},\;\xi_{t-1}^{(1)}) ,\; T(s_{t-1},\;a_{t-1},\;\xi_{t-1}^{(2)})) \\
  &\le K_\xi(s_{t-1},\;a_{t-1})\,d_{\Xi}(\xi_{t-1}^{(1)} ,\; \xi_{t-1}^{(2)}).
\end{split}
\end{equation}
Under the assumption of bounded noise \(d_{\Xi}(\xi_{t-1},\;0)\le \sigma_\xi\), we have \(d_{\Xi}(\xi_{t-1}^{(1)} ,\; \xi_{t-1}^{(2)})\le 2\sigma_\xi\), which yields the desired result.
\end{proof}
Therefore, according to Definition~\ref{df:sim_state} and Lemma~\ref{lem:radius}, the distance between all similar state pairs is guaranteed to be bounded by \(2K_\xi\sigma_\xi\), forming a bounded region suitable for applying the local Lipschitz constraint.

Furthermore, by Definition~\ref{df:sim_state}, our similar state distribution coincides exactly with the true transition kernel \(P_{\mathrm{real}}(\,\cdot\mid s_{t-1})\), thereby ensuring fidelity to the system dynamics.

\subsection{Composite Function Based Lipschitz Constraint}
\begin{theorem}[Composite Lipschitz Constraint]\label{thm:composite_lips}
Under Assumptions~\ref{ass:transition_lips}-\ref{ass:bounded_noise}, for all $(s_{t-1},a_{t-1})\in\mathcal{S}\times\mathcal{A}$ and $\xi_{t-1}^{(1)},\xi_{t-1}^{(2)}\in\Xi$, it holds that
\begin{equation}\label{eq:composit_lips}
\begin{split}
  &d_A(f\circ T(s_{t-1},a_{t-1},\xi_{t-1}^{(1)}) ,\; f\circ T(s_{t-1},a_{t-1},\xi_{t-1}^{(2)})) \\
  &\le K\,d_S( T(s_{t-1},a_{t-1},\xi_{t-1}^{(1)}) ,\; T(s_{t-1},a_{t-1},\xi_{t-1}^{(2)})),\\
  &\forall\,\xi_1, \xi_2\in\Xi.
\end{split}
\end{equation}
Therefore, the composite function $f\circ T$ is locally Lipschitz continuous in the noise $\xi$, with local Lipschitz constant
\begin{equation}\label{eq:composit_lips_const}
K_{\mathrm{comp}}(s_{t-1},a_{t-1}) \;=\; K\,K_\xi(s_{t-1},a_{t-1}).
\end{equation}
\end{theorem}
\noindent In this theorem, \( T \) represents the transition function, and \( f \) represents the actor. 

Theorem \ref{thm:composite_lips} can be reduced to loss optimization problem such as
\begin{equation}\label{eq:def_loss_smooth}
L=\left\| \pi_{\phi}(s_t)-\mathbb{E}_{\tilde{s}_t\sim P(\cdot\mid s_{t-1})}\bigl[\pi_{\phi}(\tilde{s}_t)\bigr] \right\|_2^2.
\end{equation}
All proofs for this paper can be found in the Appendix A.
\(\mathbb{E}_{\tilde{s}_t\sim P(\cdot\mid s_{t-1})}[\pi_{\phi}(\tilde{s}_t)]\) denotes the expectation of the policy outputs over all next similar states \(\tilde{s}_t\) reachable from the predecessor state \(s_{t-1}\).
In other words, \(L\) is designed to minimize the difference between the policy output \(\pi_{\phi}(s_t)\) for a state \(s_t\) randomly sampled from the transition distribution
\(s_t \sim P(\,\cdot\mid s_{t-1}),\)
and the average policy output over the same distribution,
\(
\mathbb{E}_{\tilde{s}_t\sim P(\,\cdot\mid s_{t-1})}\bigl[\pi_{\phi}(\tilde{s}_t)\bigr].
\)
This mechanism propagates the local Lipschitz property of the transition function \(T\) through the composite function \(f\circ T\), thereby enforcing the local Lipschitz constraint on the actor network \(f\) during training.

\subsection{Action Smoothing via Predictions from Preceding States}
In this section, we introduce ASAP, a novel RL action smoothing method using loss-penalty that integrates the spatial smoothness term grounded in the similar state definition based on transition (Definition~\ref{df:sim_state}) with the temporal smoothness term proposed by Grad-CAPS \cite{lee2024gradient}. 
We define the ASAP policy loss as 
\begin{equation}  \label{eq:asap_policy_loss}
J^{\rm ASAP}_{\pi_\phi} = J_{\pi_\phi} + \lambda_S L_S + \lambda_P L_P + \lambda_T L_T,
\end{equation}
which comprises four terms.
$J_{\pi_{\phi}}$ denotes the standard RL actor loss used in algorithms such as Proximal Policy Optimization (PPO) \cite{schulman2017proximal} and Soft Actor-Critic (SAC) \cite{haarnoja2018soft}. 

The spatial loss \(L_{\rm S}\) is defined as 
\begin{equation}\label{eq:asap_loss_smooth}
    L_{\rm S}
    = \bigl\|\pi_{\phi}(s_t) - \texttt{stopgrad}(\pi_{\rm P}(s_{t-1}))\bigr\|_2^2,
\end{equation}
which, as shown in Figure~\ref{fig:sim_asap} and \ref{fig:update_procedure}, minimizes the difference between the current action \(\pi_{\phi}(s_t)\) and the next expected action \(\pi_{\rm P}(s_{t-1})\) predicted from the previous state \(s_{t-1}\).
The prediction loss \(L_{\rm P}\) is expressed as
\begin{equation}\label{eq:asap_loss_pred}
    L_{\rm P}
    = \bigl\|\pi_{\rm P}(s_{t-1}) - \texttt{stopgrad}(\pi_{\phi}(s_t))\bigr\|_2^2,
\end{equation}
and trains the predictor \(\pi_{\rm P}\) to closely mimic the policy output \(\pi_{\phi}(s_t)\), illustrated in Figure~\ref{fig:update_procedure}.
Because \(\pi_{\phi}\) generates actual actions while \(\pi_{P}\) provides target action expectations, they differ in objectives, learning signals, and influencing factors and thus require separate learning strengths.
Therefore, we fix the target values, divide the loss into Eq.~\ref{eq:asap_loss_smooth} and Eq.~\ref{eq:asap_loss_pred}, and assign distinct strength parameters to each component.

\(L_T\) is the temporal Lipschitz penalty directly adopted from the Grad-CAPS~\cite{lee2024gradient} temporal loss and defined as 
\begin{equation}\label{eq:gradcaps_loss}
L_{\rm T} \;=\; \Bigl\|\frac{a_{t+1} - 2a_t + a_{t-1}}{\tanh(a_{t+1} - a_{t-1}) + \epsilon}\Bigr\|_2^2,
\end{equation}
where $\epsilon$ is small positive constant added for numerical stability to avoid division by zero.
\(L_{\rm T}\) constrains the second‐order change in action over time, and the denominator $\tanh(a_{t+1} - a_{t-1})$ provides robustness to variations in the action scale.
The second-order difference enables more flexible action changes while maintaining stability.
The regularization strength of each term is controlled by the hyperparameters \(\lambda_S, \lambda_P, \lambda_T > 0\).

\subsubsection{Training Method.}
Predictor training faces a ``moving target'' issue, as shifting outputs hinder smoothness optimization and can degrade returns. 
Although soft updates help stabilize training \cite{lillicrap2015continuous,fujimoto2018addressing}, their lag in adapting to rapidly changing distributions can impair performance in on-policy methods like PPO. 
This issue can be addressed by increasing the number of parallel environments to collect more data under the same policy and by reducing the predictor loss weight \(\lambda_P\).

\section{Experiment}

\subsection{Experimental Setup}
\subsubsection{Simulation Framework.}
We evaluated ASAP using two simulation frameworks: \textbf{Gymnasium} \cite{towers2024gymnasium} and \textbf{Isaac‑Lab} \cite{mittal2023orbit}.
Gymnasium, integrated with the MuJoCo \cite{todorov2012mujoco} is a widely adopted benchmark for continuous RL; we used it to empirically validate ASAP's theoretical foundations and evaluate performance. 
To assess applicability to realistic robotic scenarios, we assessed our method on Isaac‑Lab, a high-fidelity simulator.

\subsubsection{Implementation Details.}
All Gymnasium experiments were implemented using Stable Baselines3 (SB3) \cite{stable-baselines3}.
The experiments were conducted under both PPO and SAC settings, with baselines including original SAC/PPO, CAPS \cite{mysore2021regularizing}, L2C2 \cite{kobayashi2022l2c2}, and Grad-CAPS \cite{lee2024gradient} (hereafter GRAD).
Isaac-Lab experiments were built on rl-games \cite{rl-games2021} with domain randomization and observation noise for realism, and were compared against original PPO. 
We fixed preset random seeds for training and evaluation. 
Full details of the experiment setup are provided in Appendix~B.
Figure~\ref{fig:arch} shows our ASAP implementation.
To implement the predictor, we added a prediction head to the actor's MLP in addition to the action head.
The action head is trained with \(L_S\) (Eq.~\ref{eq:asap_loss_smooth}), while the prediction head is trained with \(L_P\) (Eq.~\ref{eq:asap_loss_pred}).

\subsubsection{Evaluation Metrics.}
Quantitative evaluation uses two metrics: \textbf{Cumulative Return (\emph{re})} and \textbf{Smoothness Score (\emph{sm})}.
The cumulative return serves as a standard metric for assessing overall task performance.
To quantify action oscillations, we adopted a smoothness measurement method based on the FFT frequency spectrum introduced in prior work \cite{mysore2021regularizing, christmann2024benchmarking} :
\begin{equation}\label{eq:smoothness}
Sm \;=\; \frac{2}{n\,f_s} \sum_{i=1}^{n} M_i\,f_i,
\end{equation}
where $f_i$ and $M_i$ denote the frequency and amplitude of the $i$-th component, and $f_s$ denotes the sampling frequency. 
This metric computes a weighted average over the frequency components, reflecting both their magnitudes and frequencies.  
Higher values indicate the presence of more high-frequency components, while lower values indicate fewer high-frequency components and smoother motion.

\begin{table}
    \centering
    {\fontsize{9pt}{\baselineskip}\selectfont
    \renewcommand{\tabularxcolumn}[1]{m{#1}}
    \begin{tabularx}{\linewidth}{
        >{\centering\arraybackslash}m{2.0cm}
        *{2}{>{\centering\arraybackslash}X>{\centering\arraybackslash\columncolor[gray]{.97}}X}
    }
    \toprule
    \multirow{2}{*}{Environment}
      & \multicolumn{2}{c}{SAC(Base)}
      & \multicolumn{2}{c}{SAC(Predictor)} \\
    \cmidrule(lr){2-3}\cmidrule(lr){4-5}
      & \emph{re}\,$\uparrow$ & \emph{sm}\,$\downarrow$
      & \emph{re}\,$\uparrow$ & \emph{sm}\,$\downarrow$ \\
    \midrule
    LunarLander      & \shortstack{280.1\\(19.5)} & \shortstack{0.296\\(0.063)}
             & \shortstack{278.8\\(20.2)} & \shortstack{0.227\\(0.041)} \\
    Reacher & \shortstack{-3.61\\(1.40)}   & \shortstack{0.051\\(0.017)}
             & \shortstack{-3.57\\(1.42)} & \shortstack{0.048\\(0.015)} \\
    Hopper & \shortstack{3330\\(412)}   & \shortstack{0.857\\(0.061)}
             & \shortstack{3478\\(182)} & \shortstack{0.712\\(0.085)} \\
    Walker & \shortstack{4467\\(435)}   & \shortstack{0.836\\(0.131)}
             & \shortstack{4419\\(641)} & \shortstack{0.715\\(0.138)} \\
    \bottomrule
    \end{tabularx}
    }
    \caption{Cumulative return (\emph{re}) and smoothness score (\emph{sm}) of SAC before and after fine-tuning with predictor across Gymnasium. Standard deviations shown in parentheses.}
    \label{tab:feasibility}
\end{table}

\begin{table*}[t]
    \centering
    {\fontsize{9pt}{\baselineskip}\selectfont
    \renewcommand{\tabularxcolumn}[1]{m{#1}}
    \begin{tabularx}{\linewidth}{
        >{\centering\arraybackslash}m{1.3cm}
        *{6}{>{\centering\arraybackslash}X>{\centering\arraybackslash\columncolor[gray]{.97}}X}
    }
    \toprule
    \multirow{2}{*}{Method}
      & \multicolumn{2}{c}{LunarLander}
      & \multicolumn{2}{c}{Pendulum}
      & \multicolumn{2}{c}{Reacher}
      & \multicolumn{2}{c}{Ant}
      & \multicolumn{2}{c}{Hopper}
      & \multicolumn{2}{c}{Walker} \\
    \cmidrule(lr){2-3}\cmidrule(lr){4-5}\cmidrule(lr){6-7}
    \cmidrule(lr){8-9}\cmidrule(lr){10-11}\cmidrule(lr){12-13}
      & \emph{re}\,$\uparrow$ & \emph{sm}\,$\downarrow$
      & \emph{re}\,$\uparrow$ & \emph{sm}\,$\downarrow$
      & \emph{re}\,$\uparrow$ & \emph{sm}\,$\downarrow$
      & \emph{re}\,$\uparrow$ & \emph{sm}\,$\downarrow$
      & \emph{re}\,$\uparrow$ & \emph{sm}\,$\downarrow$
      & \emph{re}\,$\uparrow$ & \emph{sm}\,$\downarrow$ \\
    \midrule
PPO Base & 
  \shortstack{172.8\\(100.2)} & \shortstack{0.309\\(0.071)} &
  \shortstack{\textbf{-169.4}\\(93.4)} & \shortstack{0.383\\(0.102)} &
  \shortstack{-5.67\\(2.12)} & \shortstack{0.034\\(0.015)} &
  \shortstack{1434\\(634)} & \shortstack{1.497\\(0.391)} &
  \shortstack{\textbf{2902}\\(929)} & \shortstack{1.709\\(0.524)} &
  \shortstack{\underline{2654}\\(1127)} & \shortstack{1.764\\(0.511)} \\

CAPS & 
  \shortstack{187.1\\(104.0)} & \shortstack{\underline{0.246}\\(0.054)} &
  \shortstack{-238.2\\(172.6)} & \shortstack{\underline{0.323}\\(0.105)} &
  \shortstack{\underline{-5.20}\\(1.83)} & \shortstack{\textbf{0.028}\\(0.014)} &
  \shortstack{2185\\(837)} & \shortstack{1.259\\(0.253)} &
  \shortstack{2362\\(981)} & \shortstack{0.281\\(0.057)} &
  \shortstack{2179\\(1154)} & \shortstack{0.565\\(0.111)} \\

L2C2 & 
  \shortstack{\underline{197.6}\\(95.5)} & \shortstack{0.292\\(0.069)} &
  \shortstack{-203.0\\(164.7)} & \shortstack{0.387\\(0.094)} &
  \shortstack{\textbf{-4.67}\\(1.93)} & \shortstack{0.039\\(0.015)} &
  \shortstack{1393\\(749)} & \shortstack{1.459\\(0.425)} &
  \shortstack{2345\\(1048)} & \shortstack{1.344\\(0.631)} &
  \shortstack{2014\\(1102)} & \shortstack{1.686\\(0.604)} \\

GRAD & 
  \shortstack{188.0\\(105.6)} & \shortstack{0.298\\(0.070)} &
  \shortstack{-199.1\\(161.4)} & \shortstack{0.336\\(0.119)} &
  \shortstack{-5.48\\(2.16)} & \shortstack{0.031\\(0.015)} &
  \shortstack{\underline{2419}\\(804)} & \shortstack{\underline{1.121}\\(0.172)} &
  \shortstack{\underline{2737}\\(901)} & \shortstack{\underline{0.193}\\(0.034)} &
  \shortstack{1967\\(1189)} & \shortstack{\textbf{0.342}\\(0.082)} \\

ASAP & 
  \shortstack{\textbf{200.7}\\(105.6)} & \shortstack{\textbf{0.221}\\(0.061)} &
  \shortstack{\underline{-178.6}\\(100.8)} & \shortstack{\textbf{0.318}\\(0.119)} &
  \shortstack{-5.44\\(1.99)} & \shortstack{\textbf{0.028}\\(0.013)} &
  \shortstack{\textbf{2574}\\(814)} & \shortstack{\textbf{1.092}\\(0.215)} &
  \shortstack{2691\\(1004)} & \shortstack{\textbf{0.179}\\(0.034)} &
  \shortstack{\textbf{3128}\\(1045)} & \shortstack{\underline{0.345}\\(0.094)} \\
    \bottomrule
    \end{tabularx}
        }
    \caption{Cumulative return (\emph{re}) and smoothness score (\emph{sm}) on Gymnasium benchmark under PPO setting. Higher \emph{re} and lower \emph{sm} are better. Bold indicates best, and underlined the second-best, per environment. Standard deviations shown in parentheses.} \label{tab:ppo_metrics}
\end{table*}

\begin{table*}[t]
    \centering
    {\fontsize{9pt}{\baselineskip}\selectfont
    \renewcommand{\tabularxcolumn}[1]{m{#1}}
    \begin{tabularx}{\linewidth}{
        >{\centering\arraybackslash}m{1.3cm}
        *{6}{>{\centering\arraybackslash}X>{\centering\arraybackslash\columncolor[gray]{.97}}X}
    }
    \toprule
    \multirow{2}{*}{Method}
      & \multicolumn{2}{c}{LunarLander}
      & \multicolumn{2}{c}{Pendulum}
      & \multicolumn{2}{c}{Reacher}
      & \multicolumn{2}{c}{Ant}
      & \multicolumn{2}{c}{Hopper}
      & \multicolumn{2}{c}{Walker} \\
    \cmidrule(lr){2-3}\cmidrule(lr){4-5}\cmidrule(lr){6-7}
    \cmidrule(lr){8-9}\cmidrule(lr){10-11}\cmidrule(lr){12-13}
      & \emph{re}\,$\uparrow$ & \emph{sm}\,$\downarrow$
      & \emph{re}\,$\uparrow$ & \emph{sm}\,$\downarrow$
      & \emph{re}\,$\uparrow$ & \emph{sm}\,$\downarrow$
      & \emph{re}\,$\uparrow$ & \emph{sm}\,$\downarrow$
      & \emph{re}\,$\uparrow$ & \emph{sm}\,$\downarrow$
      & \emph{re}\,$\uparrow$ & \emph{sm}\,$\downarrow$ \\
    \midrule
    SAC Base &
      \shortstack{\underline{279.2}\\(18.1)} & \shortstack{0.289\\(0.052)} &
      \shortstack{\textbf{-145.6}\\(69.6)} & \shortstack{0.421\\(0.141)} &
      \shortstack{\textbf{-3.57}\\(1.41)} & \shortstack{0.051\\(0.016)} &
      \shortstack{4239\\(1552)} & \shortstack{1.715\\(0.475)} &
      \shortstack{3349\\(385)} & \shortstack{0.856\\(0.062)} &
      \shortstack{4476\\(392)} & \shortstack{0.823\\(0.128)} \\

    CAPS &
      \shortstack{251.0\\(87.2)} & \shortstack{\underline{0.259}\\(0.057)} &
      \shortstack{-149.4\\(68.9)} & \shortstack{\textbf{0.304}\\(0.119)} &
      \shortstack{\textbf{-3.57}\\(1.38)} & \shortstack{0.046\\(0.014)} &
      \shortstack{\underline{4532}\\(1526)} & \shortstack{1.739\\(0.485)} &
      \shortstack{3413\\(33)} & \shortstack{0.793\\(0.085)} &
      \shortstack{4320\\(302)} & \shortstack{0.815\\(0.081)} \\

    L2C2 &
      \shortstack{249.3\\(90.4)} & \shortstack{0.349\\(0.122)} &
      \shortstack{-146.2\\(69.1)} & \shortstack{0.383\\(0.145)} &
      \shortstack{\textbf{-3.57}\\(1.38)} & \shortstack{0.051\\(0.016)} &
      \shortstack{4125\\(1420)} & \shortstack{1.811\\(0.395)} &
      \shortstack{\textbf{3455}\\(53)} & \shortstack{0.893\\(0.092)} &
      \shortstack{\underline{4472}\\(558)} & \shortstack{0.857\\(0.183)} \\

    GRAD &
      \shortstack{250.9\\(86.7)} & \shortstack{\underline{0.248}\\(0.055)} &
      \shortstack{\textbf{-145.6}\\(69.6)} & \shortstack{0.345\\(0.123)} &
      \shortstack{-3.62\\(1.47)} & \shortstack{\textbf{0.043}\\(0.013)} &
      \shortstack{3938\\(2483)} & \shortstack{\underline{1.423}\\(0.410)} &
      \shortstack{3190\\(628)} & \shortstack{\underline{0.588}\\(0.115)} &
      \shortstack{4339\\(445)} & \shortstack{\underline{0.612}\\(0.061)} \\

    ASAP &
      \shortstack{\textbf{280.9}\\(19.5)} & \shortstack{\textbf{0.185}\\(0.038)} &
      \shortstack{-145.7\\(68.9)} & \shortstack{\underline{0.338}\\(0.107)} &
      \shortstack{-3.61\\(1.43)} & \shortstack{\underline{0.045}\\(0.015)} &
      \shortstack{\textbf{4966}\\(1467)} & \shortstack{\textbf{1.226}\\(0.267)} &
      \shortstack{\underline{3448}\\(313)} & \shortstack{\textbf{0.498}\\(0.077)} &
      \shortstack{\textbf{4665}\\(494)} & \shortstack{\textbf{0.578}\\(0.096)} \\
    \bottomrule
    \end{tabularx}
        }
    \caption{Cumulative return (\emph{re}) and smoothness score (\emph{sm}) on Gymnasium benchmark under SAC setting. Higher \emph{re} and lower \emph{sm} are better. Bold indicates best, and underlined the second-best, per environment. Standard deviations shown in parentheses.} \label{tab:sac_metrics}
\end{table*}

\begin{table*}[t]
    \centering
    {\fontsize{9pt}{\baselineskip}\selectfont
    \renewcommand{\tabularxcolumn}[1]{m{#1}}  
    \begin{tabularx}{\linewidth}{
        >{\centering\arraybackslash}m{2.1cm}
        *{4}{>{\centering\arraybackslash}X>{\centering\arraybackslash\columncolor[gray]{.97}}X}
    }
    \toprule
    \multirow{2}{*}{Method}
      & \multicolumn{2}{c}{Reach-Franka}    
      & \multicolumn{2}{c}{Lift-Cube-Franka}      
      & \multicolumn{2}{c}{Repose-Cube-Allegro}
      & \multicolumn{2}{c}{Anymal-Velocity-Rough} \\    
    \cmidrule(lr){2-3}\cmidrule(lr){4-5}\cmidrule(lr){6-7}\cmidrule(lr){8-9}
      & \emph{re}\,$\uparrow$ & \emph{sm}\,$\downarrow$
      & \emph{re}\,$\uparrow$ & \emph{sm}\,$\downarrow$
      & \emph{re}\,$\uparrow$ & \emph{sm}\,$\downarrow$
      & \emph{re}\,$\uparrow$ & \emph{sm}\,$\downarrow$\\
    \midrule
    PPO Base &
      \shortstack{0.380\\(0.233)} &
      \shortstack{0.959\\(0.113)} &
      \shortstack{\textbf{136.1}\\(15.1)} &
      \shortstack{2.315\\(0.873)} &
      \shortstack{31.90\\(13.09)} &
      \shortstack{2.658\\(0.612)} &
      \shortstack{\textbf{16.69}\\(2.66)} &
      \shortstack{3.502\\(0.285)} \\

    ASAP &
      \shortstack{\textbf{0.525}\\(0.195)} &
      \shortstack{\textbf{0.658}\\(0.065)} &
      \shortstack{134.0\\(32.2)} &
      \shortstack{\textbf{0.926}\\(0.305)} &
      \shortstack{\textbf{32.82}\\(15.91)} &
      \shortstack{\textbf{2.363}\\(0.558)} &
      \shortstack{16.09\\(2.92)} &
      \shortstack{\textbf{2.861}\\(0.306)} \\
    \bottomrule
    \end{tabularx}
    }
    \caption{Cumulative return (\emph{re}) and smoothness score (\emph{sm}) for PPO Base and PPO + ASAP on the Isaac-Lab environment. Higher \emph{re} and lower \emph{sm} are better. Bold indicates the best per environment. Standard deviations shown in parentheses.}
    \label{tab:isaac_metrics}
\end{table*}

\begin{table}[t]
    \centering
    {\fontsize{9pt}{\baselineskip}\selectfont
    \renewcommand{\tabularxcolumn}[1]{m{#1}}
    \begin{tabularx}{\linewidth}{
        >{\centering\arraybackslash}m{2.0cm}
        *{2}{>{\centering\arraybackslash}X>{\centering\arraybackslash\columncolor[gray]{.97}}X}
    }
    \toprule
    \multirow{2}{*}{Method}
      & \multicolumn{2}{c}{Hopper}
      & \multicolumn{2}{c}{Walker} \\
    \cmidrule(lr){2-3}\cmidrule(lr){4-5}
      & \emph{re}\,$\uparrow$ & \emph{sm}\,$\downarrow$
      & \emph{re}\,$\uparrow$ & \emph{sm}\,$\downarrow$ \\
    \midrule
    LipsNet &
      \shortstack{2798\\(665)} & \shortstack{0.838\\(0.143)} &
      \shortstack{3942\\(462)} & \shortstack{0.915\\(0.178)} \\

    Lips+CAPS &
      \shortstack{3096\\(21)} & \shortstack{0.650\\(0.046)} &
      \shortstack{3464\\(1377)} & \shortstack{0.665\\(0.143)} \\

    Lips+ASAP &
      \shortstack{\textbf{3105}\\(45)} & \shortstack{\textbf{0.381}\\(0.021)} &
      \shortstack{\textbf{4475}\\(542)} & \shortstack{\textbf{0.485}\\(0.103)} \\
    \bottomrule
    \end{tabularx}
        }
    \caption{Measurements of \emph{re} and \emph{sm} in SAC setting when combined with LipsNet.} \label{tab:comparison_lipsnet}
\end{table}

\subsection{Effectiveness on Transition-induced Similar State}

We conducted an experiment to empirically validate the effectiveness of the proposed transition-induced similar state (Definition~\ref{df:sim_state}). 
Specifically, a SAC policy was fine-tuned using a predictor trained via supervised learning on a fixed replay buffer. 
The details are described in Appendix B.

As shown in Table~\ref{tab:feasibility}, this fine-tuning procedure consistently reduced the smoothness score, with an average improvement of 15.1\%.
In particular, environments with more pronounced high-frequency oscillations, such as Hopper and Walker, exhibited reductions of up to 16.9\% and 14.4\%, respectively. 
Meanwhile, the cumulative return remained stable without significant changes.
These results suggested that the transition-induced similar state effectively suppressed action oscillations while preserving task performance.

\subsection{Evaluation on Gymnasium}

Table~\ref{tab:ppo_metrics} presents the results under the PPO setting.  
ASAP achieved the highest cumulative return in 3 out of 6 environments and recorded the best smoothness score in 5 environments.
On average, the proposed method reduced \emph{sm} by 43.3\% compared to the PPO baseline, with particularly large improvements observed in Hopper (89.5\%) and Walker (80.4\%).  
In the Hopper environment, the slight decrease in \emph{re} appears to result from excessive spatial smoothing in regions requiring rapid action changes. 
Nonetheless, our approach minimized the reduction in \emph{re}, preserving reward stability. 
In contrast, CAPS and L2C2 reduced \emph{sm} in Hopper by 83.5\% and 21.3\%, respectively, both of which are smaller than the reduction achieved by the proposed method, and they experienced a larger drop in \emph{re}. 
This suggested that the heuristic similar state definitions employed by these methods failed to form neighborhood structures that adequately satisfy the optimal Lipschitz constraint.

ASAP also showed strong performance in the SAC setting (Table~\ref{tab:sac_metrics}), achieving the highest \emph{re} in 3 out of 6 environments and the best \emph{sm} in 4 environments.
On average, the proposed method reduced \emph{sm} by 27.9\% compared to the SAC baseline, while maintaining a similar level of \emph{re}.  
Notably, our approach consistently matched or outperformed Grad-CAPS in both \emph{re} and \emph{sm}, suggesting that the Grad-based temporal term and our spatial regularization in ASAP complement each other effectively.  
These results demonstrated that the proposed method achieves a desirable balance between high reward and smoothness.

\begin{table}[t]
    \centering
    {\fontsize{9pt}{\baselineskip}\selectfont
    \renewcommand{\tabularxcolumn}[1]{m{#1}}
    \begin{tabularx}{\linewidth}{
        >{\centering\arraybackslash}m{1.0cm}
        >{\centering\arraybackslash}m{1.0cm}
        *{2}{>{\centering\arraybackslash}X>{\centering\arraybackslash\columncolor[gray]{.97}}X}
    }
    \toprule
    \multirow{2}{*}{$L_S$} 
      & \multirow{2}{*}{$L_T$} 
      & \multicolumn{2}{c}{Hopper} 
      & \multicolumn{2}{c}{Walker} \\
    \cmidrule(lr){3-4}\cmidrule(lr){5-6}
      & 
      & \emph{re}\,$\uparrow$ & \emph{sm}\,$\downarrow$ 
      & \emph{re}\,$\uparrow$ & \emph{sm}\,$\downarrow$ \\
    \midrule
    -      & GRAD   & \shortstack{\textbf{2963}\\(690)}  & \shortstack{0.241\\(0.051)}
                    & \shortstack{\underline{2659}\\(1176)} & \shortstack{0.541\\(0.188)} \\
    CAPS   & GRAD   & \shortstack{2264\\(1033)} & \shortstack{\underline{0.201}\\(0.044)}
                    & \shortstack{2303\\(1132)} & \shortstack{\underline{0.467}\\(0.118)} \\
    L2C2   & GRAD   & \shortstack{\underline{2925}\\(705)}  & \shortstack{0.227\\(0.035)}
                    & \shortstack{2500\\(980)}  & \shortstack{0.519\\(0.151)} \\
    ASAP   & GRAD   & \shortstack{2691\\(1004)} & \shortstack{\textbf{0.179}\\(0.034)}
                    & \shortstack{\textbf{3128}\\(1045)} & \shortstack{\textbf{0.345}\\(0.094)} \\
    \bottomrule
    \end{tabularx}
    }
    \caption{Measurements of \emph{re} and \emph{sm} when the temporal term is fixed to GRAD and combined with various spatial terms.}
    \label{tab:ablation_spatial}
\end{table}
\subsection{Applicability in Robotic Scenarios}
As shown in Table \ref{tab:isaac_metrics}, ASAP was successfully applied to robotics scenarios, improving or maintaining cumulative return while lowering smoothness across all four tasks.
In particular, in the Franka reach task, it improved \emph{re} and decreased the \emph{sm} by 31.3\%, yielding substantial improvements in both metrics.
Figure~\ref{fig:franka_reach} depicts the magnitude of the action difference for joint 1 between consecutive steps.
The mean error between the end effector and the target also decreased by 30.3\%, from 0.90 cm to 0.62 cm, suggesting that smoothing constraints are beneficial for tasks requiring high precision and sustained stability.
In other tasks that prioritize dynamic interaction with the environment over fine-grained precision, such as Repose-Cube and Velocity, no significant reward improvement was achieved.
Nevertheless, the consistent reward performance indicated that ASAP does not compromise the agent's responsiveness in these tasks.
However, in environments such as Lift-Cube, we observed an increase in the variance of \emph{re}, 
which we attribute to the oscillation suppression process constraining exploration and resulting in inadequate learning in a subset of random seeds.

\begin{figure}
    \centering
    \includegraphics[width=0.95\linewidth]{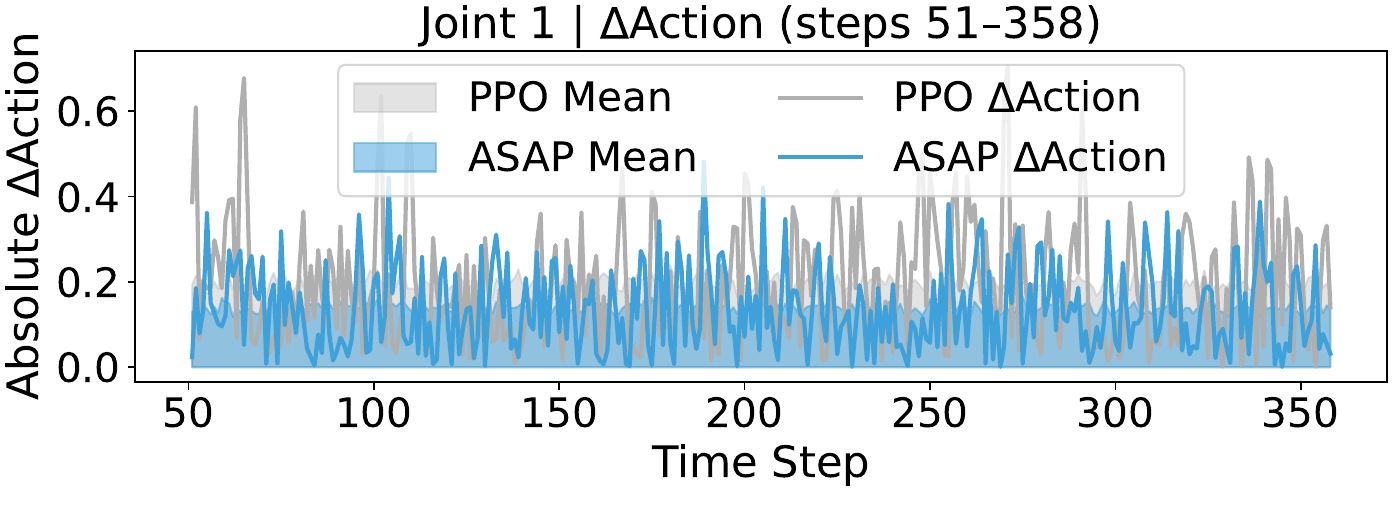}
    \caption{Action change ($\Delta$a) of Joint-1 in the Franka reach task for original PPO and ASAP. Solid line represents median $\Delta$a run; shaded area indicates mean across all seeds.}
    \label{fig:franka_reach}
\end{figure}

\subsection{Compatibility with Architectural Method}
We evaluated the compatibility of ASAP with architectural methods by integrating it with LipsNet under the SAC framework. 
For comparison, we adopted the original LipsNet~\cite{song2023lipsnet} and LipsNet+CAPS~\cite{christmann2024benchmarking} as baselines, and conducted experiments on the Hopper and Walker environments. 
As shown in Table~\ref{tab:comparison_lipsnet}, our method significantly improved the \emph{re} and \emph{sm} of LipsNet across all environments. 
In particular, the smoothness metric \emph{sm} was reduced by 54.5\% on Hopper and by 47.0\% on Walker, indicating a substantial improvement in action stability. 
Compared to the LipsNet+CAPS variant, our approach achieved 41.4\% and 27.1\% lower \emph{sm} on Hopper and Walker, respectively. 
These results demonstrated that ASAP can synergize with the architectural method.

\subsection{Ablation Study}
We verified that our spatial term represents the most effective strategy and evaluated its performance when combined with various temporal terms.
The experiments were conducted using PPO as the baseline algorithm and were limited to the Walker and Hopper environments.

\subsubsection{Comparison with Other Spatial Methods.}
We conducted experiments under four settings, fixing the temporal term to Grad-CAPS, as shown in Table~\ref{tab:ablation_spatial}.  
The CAPS spatial term reduced \emph{sm} by 15.1\% but caused a significant drop in \emph{re}. 
The L2C2 term showed minimal effect, reducing \emph{sm} by only 4.9\% with a slight decrease in \emph{re}. 
In contrast, our spatial term reduced \emph{sm} by 30.9 \%, achieving the largest oscillation reduction, while stably maintaining \emph{re}.
These results demonstrated that when the Grad-CAPS temporal term is combined with our proposed spatial term, it yields superior performance in reducing oscillations and preserving rewards compared to other methods.

\begin{table}[t]
    \centering
    {\fontsize{9pt}{\baselineskip}\selectfont
    \renewcommand{\tabularxcolumn}[1]{m{#1}}
    \begin{tabularx}{\linewidth}{
        >{\centering\arraybackslash}m{1.0cm}
        >{\centering\arraybackslash}m{1.0cm}
        *{2}{>{\centering\arraybackslash}X>{\centering\arraybackslash\columncolor[gray]{.97}}X}
    }
    \toprule
    \multirow{2}{*}{$L_S$} 
      & \multirow{2}{*}{$L_T$} 
      & \multicolumn{2}{c}{Hopper} 
      & \multicolumn{2}{c}{Walker} \\
    \cmidrule(lr){3-4}\cmidrule(lr){5-6}
      & 
      & \emph{re}\,$\uparrow$ & \emph{sm}\,$\downarrow$ 
      & \emph{re}\,$\uparrow$ & \emph{sm}\,$\downarrow$ \\
    \midrule
    -      & -      & \shortstack{2902\\(929)}  & \shortstack{1.709\\(0.524)}
                    & \shortstack{2654\\(1127)} & \shortstack{1.764\\(0.511)} \\
    -      & CAPS   & \shortstack{2962\\(797)}  & \shortstack{0.371\\(0.073)}
                    & \shortstack{2249\\(1174)} & \shortstack{0.872\\(0.209)} \\
    ASAP   & CAPS   & \shortstack{\textbf{2973}\\(706)}  & \shortstack{\underline{0.235}\\(0.047)}
                    & \shortstack{2357\\(1282)} & \shortstack{\underline{0.389}\\(0.076)} \\
    -      & GRAD   & \shortstack{\underline{2963}\\(690)}  & \shortstack{0.241\\(0.051)}
                    & \shortstack{\underline{2659}\\(1176)} & \shortstack{0.541\\(0.188)} \\
    ASAP   & GRAD   & \shortstack{2691\\(1004)} & \shortstack{\textbf{0.179}\\(0.034)}
                    & \shortstack{\textbf{3128}\\(1045)} & \shortstack{\textbf{0.345}\\(0.094)} \\
    \bottomrule
    \end{tabularx}
    }
    \caption{Measurements of \emph{re} and \emph{sm} when the spatial term is fixed to ASAP and combined with various temporal terms.}
    \label{tab:ablation_temporal}
\end{table}
\subsubsection{Compatibility with Temporal Methods.} We conducted experiments under five combination settings, as shown in Table~\ref{tab:ablation_temporal}, to evaluate whether our proposed spatial term is compatible with various temporal terms.  
When integrated with the CAPS temporal regularization, \emph{sm} exhibited substantial improvements of 36.6\% in Hopper and 55.3\% in Walker, accompanied by consistent gains in \emph{re} across both environments.
 Similarly, incorporating the Grad-CAPS temporal term led to \emph{sm} improvements of 25.7\% in Hopper and 36.2\% in Walker, while \emph{re} also improved, except for a marginal decline in Hopper.
These results showed that when our spatial term is combined with a temporal term, it reduces oscillations without significant interference.

\section{Conclusion and Limitation}
\subsubsection{Conclusion.}
In this paper, we introduced ASAP, a novel action-smoothing method for reinforcement learning. 
ASAP mitigates action oscillations by leveraging the transition function's Lipschitz continuity and penalizing second-order action differences.
Through extensive experiments, we demonstrated that our method effectively maintains policy performance while improving action smoothness.
\subsubsection{Limitation.}
While ASAP defines similar states based on observation noise to better reflect real transition dynamics, the presence of high noise levels may enlarge the resulting neighborhood, potentially diminishing the effectiveness of the local Lipschitz constraint and spatial regularization.
This issue can be mitigated by adjusting the spatial loss weight $\lambda_S$.

\section{Acknowledgements}
This work was supported by the Technology Innovation Program(RS-2025-25453780, Development of a National Humanoid AI Robot Foundation Model for Multi‑Task Applications) funded by the Ministry of Trade Industry \& Energy(MOTIE, Korea), and in part by the Institute of Information and Communications Technology Planning and Evaluation (IITP) grant funded by the Korean government (MSIT) under Grant RS-2022-00155911, Artificial Intelligence Convergence Innovation Human Resources Development (Kyung Hee University), and in part by Convergence security core talent training business support program under Grant
IITP-2023-RS-2023-00266615).


\bibliography{aaai2026}

@inproceedings{mysore2021regularizing,
  title={Regularizing action policies for smooth control with reinforcement learning},
  author={Mysore, Siddharth and Mabsout, Bassel and Mancuso, Renato and Saenko, Kate},
  booktitle={2021 IEEE International Conference on Robotics and Automation (ICRA)},
  pages={1810--1816},
  year={2021},
  organization={IEEE}
}

@inproceedings{kobayashi2022l2c2,
  title={L2c2: Locally lipschitz continuous constraint towards stable and smooth reinforcement learning},
  author={Kobayashi, Taisuke},
  booktitle={2022 IEEE/RSJ International Conference on Intelligent Robots and Systems (IROS)},
  pages={4032--4039},
  year={2022},
  organization={IEEE}
}

@inproceedings{song2023lipsnet,
  title={LipsNet: A smooth and robust neural network with adaptive Lipschitz constant for high accuracy optimal control},
  author={Song, Xujie and Duan, Jingliang and Wang, Wenxuan and Li, Shengbo Eben and Chen, Chen and Cheng, Bo and Zhang, Bo and Wei, Junqing and Wang, Xiaoming Simon},
  booktitle={International Conference on Machine Learning},
  pages={32253--32272},
  year={2023},
  organization={PMLR}
}

@inproceedings{
song2025lipsnet,
title={LipsNet++: Unifying Filter and Controller into a Policy Network},
author={Xujie Song and Liangfa Chen and Tong Liu and Wenxuan Wang and Yinuo Wang and Shentao Qin and Yinsong Ma and Jingliang Duan and Shengbo Eben Li},
booktitle={Forty-second International Conference on Machine Learning},
year={2025},
url={https://openreview.net/forum?id=KZo2XhcSg6}
}

@inproceedings{chen2021addressing,
  title={Addressing action oscillations through learning policy inertia},
  author={Chen, Chen and Tang, Hongyao and Hao, Jianye and Liu, Wulong and Meng, Zhaopeng},
  booktitle={Proceedings of the AAAI Conference on Artificial Intelligence},
  volume={35},
  number={8},
  pages={7020--7027},
  year={2021}
}

@inproceedings{lee2024gradient,
  title={Gradient-based Regularization for Action Smoothness in Robotic Control with Reinforcement Learning},
  author={Lee, I and Cao, Hoang-Giang and Dao, Cong-Tinh and Chen, Yu-Cheng and Wu, I-Chen},
  booktitle={2024 IEEE/RSJ International Conference on Intelligent Robots and Systems (IROS)},
  pages={603--610},
  year={2024},
  organization={IEEE}
}

@inproceedings{christmann2024benchmarking,
  title={Benchmarking Smoothness and Reducing High-Frequency Oscillations in Continuous Control Policies},
  author={Christmann, Guilherme and Luo, Ying-Sheng and Mandala, Hanjaya and Chen, Wei-Chao},
  booktitle={2024 IEEE/RSJ International Conference on Intelligent Robots and Systems (IROS)},
  pages={627--634},
  year={2024},
  organization={IEEE}
}

@inproceedings{gogianu2021spectral,
  title={Spectral normalisation for deep reinforcement learning: an optimisation perspective},
  author={Gogianu, Florin and Berariu, Tudor and Rosca, Mihaela C and Clopath, Claudia and Busoniu, Lucian and Pascanu, Razvan},
  booktitle={International Conference on Machine Learning},
  pages={3734--3744},
  year={2021},
  organization={PMLR}
}

@article{takase2022stability,
  title={Stability-certified reinforcement learning control via spectral normalization},
  author={Takase, Ryoichi and Yoshikawa, Nobuyuki and Mariyama, Toshisada and Tsuchiya, Takeshi},
  journal={Machine Learning with Applications},
  volume={10},
  pages={100409},
  year={2022},
  publisher={Elsevier}
}

@ARTICLE{9904958,
  author={Wang, Xu and Wang, Sen and Liang, Xingxing and Zhao, Dawei and Huang, Jincai and Xu, Xin and Dai, Bin and Miao, Qiguang},
  journal={IEEE Transactions on Neural Networks and Learning Systems}, 
  title={Deep Reinforcement Learning: A Survey}, 
  year={2024},
  volume={35},
  number={4},
  pages={5064-5078},
  doi={10.1109/TNNLS.2022.3207346}}

@INPROCEEDINGS{8202133,
  author={Tobin, Josh and Fong, Rachel and Ray, Alex and Schneider, Jonas and Zaremba, Wojciech and Abbeel, Pieter},
  booktitle={2017 IEEE/RSJ International Conference on Intelligent Robots and Systems (IROS)}, 
  title={Domain randomization for transferring deep neural networks from simulation to the real world}, 
  year={2017},
  volume={},
  number={},
  pages={23-30},
  doi={10.1109/IROS.2017.8202133}}

@article{williams1992simple,
  title={Simple statistical gradient-following algorithms for connectionist reinforcement learning},
  author={Williams, Ronald J},
  journal={Machine learning},
  volume={8},
  pages={229--256},
  year={1992},
  publisher={Springer}
}

@article{schulman2017proximal,
  title={Proximal policy optimization algorithms},
  author={Schulman, John and Wolski, Filip and Dhariwal, Prafulla and Radford, Alec and Klimov, Oleg},
  journal={arXiv preprint arXiv:1707.06347},
  year={2017}
}

@inproceedings{haarnoja2018soft,
  title={Soft actor-critic: Off-policy maximum entropy deep reinforcement learning with a stochastic actor},
  author={Haarnoja, Tuomas and Zhou, Aurick and Abbeel, Pieter and Levine, Sergey},
  booktitle={International conference on machine learning},
  pages={1861--1870},
  year={2018},
  organization={Pmlr}
}

@article{lillicrap2015continuous,
  title={Continuous control with deep reinforcement learning},
  author={Lillicrap, Timothy P and Hunt, Jonathan J and Pritzel, Alexander and Heess, Nicolas and Erez, Tom and Tassa, Yuval and Silver, David and Wierstra, Daan},
  journal={arXiv preprint arXiv:1509.02971},
  year={2015}
}

@inproceedings{fujimoto2018addressing,
  title={Addressing function approximation error in actor-critic methods},
  author={Fujimoto, Scott and Hoof, Herke and Meger, David},
  booktitle={International conference on machine learning},
  pages={1587--1596},
  year={2018},
  organization={PMLR}
}

@article{levine2016end,
  title={End-to-end training of deep visuomotor policies},
  author={Levine, Sergey and Finn, Chelsea and Darrell, Trevor and Abbeel, Pieter},
  journal={Journal of Machine Learning Research},
  volume={17},
  number={39},
  pages={1--40},
  year={2016}
}

@book{sontag1998mct,
  author    = {Sontag, Eduardo D.},
  title     = {Mathematical Control Theory: Deterministic Finite Dimensional Systems},
  edition   = {2nd},
  publisher = {Springer},
  year      = {1998}
}

@book{spong2006robot,
  author    = {Spong, M. W. and Hutchinson, S. and Vidyasagar, M.},
  title     = {Robot Modeling and Control},
  publisher = {Wiley},
  year      = {2006}
}

@book{featherstone2008rigid,
  author    = {Featherstone, R.},
  title     = {Rigid Body Dynamics Algorithms},
  publisher = {Springer},
  year      = {2008}
}

@book{camacho2013mpc,
  author    = {Camacho, E. F. and Alba, C. B.},
  title     = {Model Predictive Control},
  publisher = {Springer Science \& Business Media},
  year      = {2013}
}

@INPROCEEDINGS{572756,
  author={Doyle, J.},
  booktitle={Proceedings of 35th IEEE Conference on Decision and Control}, 
  title={Robust and optimal control}, 
  year={1996},
  volume={2},
  number={},
  pages={1595-1598 vol.2},
  doi={10.1109/CDC.1996.572756}}

@article{DELAPRESILLA2023108805,
title = {Oscillating rolling element bearings: A review of tribotesting and analysis approaches},
journal = {Tribology International},
volume = {188},
pages = {108805},
year = {2023},
issn = {0301-679X},
doi = {https://doi.org/10.1016/j.triboint.2023.108805},
url = {https://www.sciencedirect.com/science/article/pii/S0301679X23005935},
author = {Román {de la Presilla} and Sebastian Wandel and Matthias Stammler and Markus Grebe and Gerhard Poll and Sergei Glavatskih}
}

@article{towers2024gymnasium,
  title={Gymnasium: A standard interface for reinforcement learning environments},
  author={Towers, Mark and Kwiatkowski, Ariel and Terry, Jordan and Balis, John U and De Cola, Gianluca and Deleu, Tristan and Goulao, Manuel and Kallinteris, Andreas and Krimmel, Markus and KG, Arjun and others},
  journal={arXiv preprint arXiv:2407.17032},
  year={2024}
}

@article{stable-baselines3,
  author  = {Antonin Raffin and Ashley Hill and Adam Gleave and Anssi Kanervisto and Maximilian Ernestus and Noah Dormann},
  title   = {Stable-Baselines3: Reliable Reinforcement Learning Implementations},
  journal = {Journal of Machine Learning Research},
  year    = {2021},
  volume  = {22},
  number  = {268},
  pages   = {1-8},
  url     = {http://jmlr.org/papers/v22/20-1364.html}
}

@misc{rl-games2021,
title = {rl-games: A High-performance Framework for Reinforcement Learning},
author = {Makoviichuk, Denys and Makoviychuk, Viktor},
month = {May},
year = {2021},
publisher = {GitHub},
journal = {GitHub repository},
howpublished = {\url{https://github.com/Denys88/rl_games}},
}

@inproceedings{todorov2012mujoco,
  title={MuJoCo: A physics engine for model-based control},
  author={Todorov, Emanuel and Erez, Tom and Tassa, Yuval},
  booktitle={2012 IEEE/RSJ International Conference on Intelligent Robots and Systems},
  pages={5026--5033},
  year={2012},
  organization={IEEE},
  doi={10.1109/IROS.2012.6386109}
}

@article{mittal2023orbit,
   author={Mittal, Mayank and Yu, Calvin and Yu, Qinxi and Liu, Jingzhou and Rudin, Nikita and Hoeller, David and Yuan, Jia Lin and Singh, Ritvik and Guo, Yunrong and Mazhar, Hammad and Mandlekar, Ajay and Babich, Buck and State, Gavriel and Hutter, Marco and Garg, Animesh},
   journal={IEEE Robotics and Automation Letters},
   title={Orbit: A Unified Simulation Framework for Interactive Robot Learning Environments},
   year={2023},
   volume={8},
   number={6},
   pages={3740-3747},
   doi={10.1109/LRA.2023.3270034}
}

\end{document}